\documentclass[]{article}

\usepackage{fancyhdr} 
\usepackage{color} 
\usepackage{hyperref} 
\usepackage{graphicx} 
\usepackage{amsmath}
\usepackage{amsthm}
\newtheorem{theorem}{Theorem}
\newtheorem{definition}{Definition}
\newtheorem{proposition}{Proposition}
\newtheorem{lemma}{Lemma}

\newtheorem{remark}{Remark}
\newcommand{\I}{\mathbbm{1}}
\usepackage{amsmath, amssymb, enumerate,url,booktabs,color,todonotes}
\usepackage{bbm}
\usepackage[normalem]{ulem}
\title{Conformal inference for regression on Riemannian Manifolds}
\author{Alejandro Cholaquidis, Fabrice Gamboa, Leonardo Moreno}

\begin{document}
	
	\maketitle
	
\begin{abstract}
Regression on manifolds, and, more broadly, statistics on manifolds, has garnered significant importance in recent years due to the vast number of applications for non Euclidean data.
Circular data is a classic example, but so is data in the space of covariance matrices, data on the Grassmannian manifold obtained as a result of principal component analysis, among many  others.  
In this work we investigate prediction sets for regression scenarios when the response variable, denoted by $Y$, resides in a manifold, and the covariable, denoted by $X$, lies in an Euclidean space.
This extends the concepts delineated in \cite{waser14} to this novel context.
Aligning with traditional principles in conformal inference, these prediction sets are distribution-free, indicating that no specific assumptions are imposed on the joint distribution of $(X,Y)$, and they maintain a non-parametric character.
We prove the asymptotic almost sure convergence of the empirical version of these regions on the manifold to their population counterparts.
The efficiency of this method is shown through a comprehensive simulation study and an analysis involving real-world data.

\end{abstract}

\section{Introduction}

Conformal prediction is a powerful set of statistical tools that operates under minimal assumptions about the underlying model. It is primarily used to construct confidence sets that are applicable to a diverse range of problems in fields such as machine learning and statistics. Is unique in its ability to construct prediction regions that guarantee a specified level of coverage for finite samples, irrespective of the distribution of the data. This is particularly crucial in high-stakes decision-making scenarios where maintaining a certain level of coverage is critical.

Unlike other methods, conformal prediction does not require strong assumptions about the sample distribution, such as normality. The aim is to construct prediction regions as small as possible to yield informative and precise predictions, enhancing the tool's utility in various applications.

The approach was first proposed by Vovk, Gammerman, and Shafer in the late 1990s, as referenced in \cite{vovk98}. Since its inception, it has been the subject of intense research activity. Originally formulated for binary classification problems, the method has since been expanded to accommodate regression, multi-class classification, functional data, functional time series, and anomaly detection, among others. Several applications of this method can be found in the book \cite{bala14}.

In the context of regression, conformal prediction has proven to be efficient in constructing prediction sets, as evidenced by works such as \cite{waser14}, \cite{waser14_2}, and \cite{kuleshov2018conformal}. To enhance the performance of these prediction sets, particularly to decrease the length of prediction intervals (when the output is one-dimensional), a combination of conformal inference and quantile regression was proposed in \cite{romano2019conformalized}. In  \cite{fong2021conformal}, computationally efficient conformal inference methods for Bayesian models were proposed. 

The method has also been extended to functional regression, where the predictors and responses are functions, rather than vectors, see for instance  \cite{lei2015conformal}, \cite{fontana20} and \cite{diqui22}. In this case, the prediction regions take the form of functional envelopes that have a high probability of containing the true function.

In the field of classification, conformal prediction has been employed to tackle a broad spectrum of problems. These include image classification, as in \cite{lei2014classification}, and text classification, as in \cite{vovk03}. For multi-class classification, a prevalent approach is to create prediction sets that have a high likelihood of encompassing the correct class. This can be realized via the use of `Venn prediction sets', as outlined in \cite{liu16}. These sets partition the label space into overlapping regions, each one corresponding to a distinct class.

In summary, conformal prediction provides  powerful statistical tools and has been successfully applied to a wide range of problems in machine learning, statistics, and related fields. Its main advantage is its ability to provide distribution-free prediction regions that can be used in the presence of any underlying distribution of the data. As this field of research  continues to evolve, it is expected to find even more applications in the future.

\section{Conformal inference on manifolds}
Although, as mentioned, it has been extended to various scenarios, there are no proposals or extensions to the case where the output is in a Riemannian manifold. This case is of cumbersome importance because there are certain types of data that, due to their inherent characteristics, must be treated as data on manifolds. A prominent example of this are covariance matrices (for instance, the volatility of a portfolio, see \cite{best10} and \cite{huang14}), which are data in the manifold of positive definite matrices (see \cite{chola23weighted}). Another example is given by Vectorcardiograms that are condensed into data on the Stiefel manifold (see \cite{chola23level}). The outcome of performing a principal component analysis results in data on the Grassmannian manifold (see \cite{hong16}), and the wind measurements can be represented as data on the cylinder (see \cite{chola21}).

Other applications in image analysis are developed in \cite{pene19} and more generally in machine learning in \cite{gui23}.

The extension to this context is not trivial since when the output \(Y\) belongs to a Riemannian manifold \(\mathcal{M}\), several arguments used in Euclidean conformal inference are not straightforwardly extendable. For instance, the Frechet mean on the manifold can be defined, but it does not always exist nor is it necessarily unique. Moreover, formulating a regression model on the manifold is not easy because of the absence of any additive structure, although there have been recent advances in this area (see for instance \cite{petersen2019frechet}).

We extend the methodologies outlined by \cite{waser14} to the broader context of pairs \((X,Y)\) where \(X \in \mathbb{R}^d\) and the response \(Y\) has support included on a, smooth enough, \(\ell\)-dimensional manifold \(\mathcal{M}\). \cite{waser14} is focused on the scenario where \(Y\) is a real-valued variable, constructing confidence sets---referred to as confidence bands---via density estimators.      One of the key tools to get the consistency of the empirical conformal region to its population counterpart, as in \cite{waser14}, is Theorem \ref{th1}. It states that the kernel-density estimator is uniformly consistent on manifolds. Once this is established, to adapt the ideas in \cite{waser14}, it must also be proved that the conditional density is uniformly bounded from above, this is done in Lemma \ref{lemaux3}.  Finally, a further distinction from \cite{waser14} is that—because the manifold may possess a nonempty boundary—points located well inside the manifold are handled differently from those situated near its edge.  Here, we obtain an upper bound for the discrepancy between the empirical confidence set \(\hat{\mathbf{C}}_n(X)\) and its theoretical counterpart \(\mathbf{C}_P(X)\), in terms of $\nu$, the volume measure on \(\mathcal{M}\). More precisely, as detailed in Theorem \ref{thmain} (see also Remark \ref{rateth2}) the probability that this discrepancy being larger than \(n^{1/((\ell+3)(d+2))}\) is bounded from above by \(A_\lambda n^{-\lambda}\), for any \(\lambda>0\), \(A_\lambda\) being a constant that depends only on \(\lambda\), where \(\ell\) denotes the dimensionality of \(\mathcal{M}\) and \(X\in \mathbb{R}^d\). Our main and stronger theorem is for compact manifolds, however, we considered in section \ref{noncomp} the case of non-compact manifolds.  \color{black} 

The kernel-based density estimator proposed in \cite{chola21} uses the Euclidean distance instead of the geodesic distance on the manifold, which may be unknown in some cases, see Equation (\ref{kernel}).      Due to the curse of dimensionality, kernel-based density estimators are generally unsuitable for high-dimensional problems. In Section~\ref{highdim}, we discuss an approach to overcome this limitation that builds on the ideas introduced in \cite{izbicki2022cd}. In essence, instead of estimating the conditional density \(p(y| x)\) locally (i.e., using only sample points close to \(x\)), the proposed method uses all sample points \((X_i,Y_i)\) whose estimated conditional densities \(\hat{p}(Y_i| X_i)\) are similar (in a sense to be defined).

Conformal inference is tackled in metric spaces in  \cite{lugosi2024}; however, given the generality, convergence rates are not provided, as in the present work, but only consistency in probability, for the case where the confidence bands are constructed through regression.

\subsection{Conformal inference in a nutshell}\label{nutshell}

In this section, we briefly provide the basic foundations of conformal inference to facilitate the reading of the following sections. For a more detailed reading see \cite{fontana2023}. A key hypothesis in conformal inference is the exchangeability assumption, which means that for any permutation $\pi$ of $\{1,\dots,n\}$ the distribution of the sample $(Z_1,\dots,Z_n)$ where $Z_i$ is a random element in some measurable space $\mathbf{Z}$ is the same as the distribution of $(Z_{\pi(1)},\dots,Z_{\pi(n)})$. Quoting \cite{fontana2023} ``A nonconformity measure  $A(B,z):\mathbf{Z}^n\times \mathbf{Z}\to \mathbb{R}$ is a way of scoring how different an example $z$ is from a bag $B=\{Z_1,\dots,Z_n\}$. Let us define, $p_{z} := |\{i = 1,\ldots,n+1 : R_i \geq R_{n+1}\}|/(n+1)$  where $R_i := A\left(\{Z_1,\ldots,Z_{i-1},Z_{i+1},\ldots,Z_n,z\}, Z_i\right)$ and $R_{n+1} := A\left(\{Z_1,\ldots,Z_n\}, z\right).$"
For $\alpha\in [0,1]$ ``we define the prediction set  \(\gamma^\alpha(Z_1,\ldots,Z_n) := \{z \in \mathbf{Z} : p_z > \alpha\}\)." The following Proposition is given in  \cite{vovk03}.

\begin{proposition}[Proposition 2.1]
	Under the exchangeability assumption, 
	$$\mathbb{P}(Z_{n+1} \not\in \gamma^\alpha(Z_1,\ldots,Z_n))\leq \alpha\quad \text{ 	for any  } \alpha.$$
\end{proposition}

In the regression setting, that is, when $Z_i=(X_i,Y_i)$ has distribution $P$, previous construction build a level set on the joint distribution (see Equation 4 in \cite{waser14}). Nevertheless, controlling
$\mathbb{P}(Y\in \mathbf{C}_P(x)|X=x)$ is much more convenient in regression (see \cite{waser14}). In this case we speak of ``conditional coverage". When there exists  a conditional density $p(y|x)$ (see hypothesis H1 in subsection \ref{assump}), the ``conditional oracle set",  $\mathbf{C}_P(x)$, (in \cite{waser14} it is called conditional oracle band), is defined as

\begin{equation}\label{oracle}
	\mathbf{C}_P(x)=\{y:p(y|x)\geq t^{\alpha}_x\},
\end{equation}
where
$t^{\alpha}(x)$ satisfies
\begin{equation}\label{level}
	\int \I_{\{p(y|x)\geq t^{\alpha}_x\}}p(y|x)dy=1-\alpha.
\end{equation}

The empirical version of $\mathbf{C}_P(x)$ obtained  from  $\aleph_n=\{(X_i,Y_i): i=1,\dots,n\}$ constituted by  an i.i.d. sample of $(X,Y)$ with distribution $P$,  is called a conditionally valid set $\mathbf{C}_n(x)$  (see \cite{waser14}).

\begin{definition}
	Given $\alpha\in (0,1)$, and $x$ in the support of $P_X$, a set $\mathbf{C}_n(x)\subset \mathcal{M}$ is said to be \textbf{conditionally valid} if
	$$\mathbb{P} \big(Y\in \mathbf{C}_n(x)| X=x\big)\geq 1-\alpha.$$
\end{definition}

This definition captures the notion of a set of possible values of $Y$ that provides a specified level of coverage for a given imput $x$.

As is shown in Lemma 1 of \cite{waser14},   non-trivial finite sample conditional validity for all $x$ in the support of $P_X$ is impossible for a continuous distribution. To overcome this limitation, the following notion of local validity is introduced.

\begin{definition} Let $\mathcal{A}=\{A_j:j\geq 1\}$ be a partition of $\textrm{supp}(P_X)$. A prediction set $\mathbf{C}_n$ is \textbf{locally valid} with respect to $\mathcal{A}$ if
	\begin{equation}\label{localval}
		\mathbb{P}(Y_{n+1}\in \mathbf{C}_n(X_{n+1})|X_{n+1}\in A_j)\geq 1-\alpha\text{ for all }j \text{ and all } P.
	\end{equation}
\end{definition}

Whenever $A_k\in \mathcal{A}$,  we will write $p(y|A_k)$ for the conditional density, w.r.t. $\nu$, of $Y$ given $X\in A_k$. We aim to prove (see Theorem \ref{thmain}) that locally valid sets converges to $\mathbf{C}_P$ when $Y$ is supported on a manifold.

\section{Assumptions}\label{assump}

In the following, $\aleph_n=\{(X_i,Y_i): i=1,\dots,n\}$ denotes an i.i.d. sample of $(X,Y)$ with distribution $P$, where $X\in \mathbb{R}^d$ and $Y\in \mathcal{M}$. Here  $(\mathcal{M},\rho)$ is  a compact $\ell$-dimensional submanifold of $\mathbb{R}^D$. The case of non compact manifolds is discussed in Section \ref{noncomp}.  We further denote by $P_X$ the marginal distribution of $X$ and by $\text{supp}(P_X)$ its support.  We denote the volume measure on $\mathcal{M}$ by $\nu$ and use $\|\cdot\|$ to denote the Euclidean norm on $\mathbb{R}^D$.  We denote by $\mu$ the $d$-dimensional Lebesgue measure in $\mathbb{R}^d$.

We will now give  the set of assumptions that we will require. H1 to H4 are also imposed in \cite{waser14}. Hypotheses H0 and H5 are imposed to guarantee the uniform convergence of the kernel-based density estimator of the conditional density.

\begin{itemize}
	\item[H0] $\mathcal{M}$ is a $\mathcal{C}^2$ submanifold, and if $\partial \mathcal{M}\neq \emptyset$, then $\partial \mathcal{M}$ is also a $\mathcal{C}^2$ submanifold.
	\item[H1] The joint distribution $P$ has a density $p_{X,Y}(x,y)$ w.r.t. $\mu\times \nu$, and the marginal distribution $P_X$ has a density $p_X$ (w.r.t. $\mu$). We denote by $p(y|x)=p_{X,Y}(x,y)/p_X(x)$ the conditional density, where $p(y|x)=0$ if $p_X(x)=0$.
	\item[H2] $p_X$ is such that that there exist $b_1,b_2$ such that  $0<b_1\leq p_X(x)\leq b_2<\infty$ for all $x$ in the support of $X$.
	\item[H3]  $p(y|x)$ is Lipschitz continuous as a function of $x$, i.e., there exists a constant $L>0$ such that $\|p(\cdot|x)-p(\cdot|x')\|_\infty \leq L|x-x'|$.
	\item[H4] There are positive constants $\epsilon_0$, $\gamma$, $c_1$, and $c_2$ such that for all $x$ in the support of $X$,
	$$c_1\epsilon^\gamma \leq P(\{y:|p(y|x)-t_x^{\alpha}|<\epsilon\}|X=x)\leq c_2\epsilon^\gamma,$$
	for all $\epsilon\leq \epsilon_0$, where $t_x^\alpha$ is given by \eqref{level}. Moreover $\inf_x t_x^{(\alpha)}\geq t_0>0$.
	\item[H5]      $p(y|A_k)$ has  continuous partial second derivatives, for short $p(y|A_k)\in \mathcal{C}^2$.
\end{itemize}

The hypothesis H0 is satisfied by a very wide range of manifolds, among them, the Stieffel and  Grassmannian Manifold. The cone of positive definite matrices, among many others. It is not a restrictive hypothesis in applications.

Assumptions  H1 to H5 impose regularity on the joint and conditional densities of the data. In particular, the assumption of Lipschitz continuity H3 ensures that small changes in the input $x$ lead to small changes in the output $y$. Assumption H2 implies that $X$ es compactly supported. Regarding Assumption H4, quoting \cite{waser14}, ``is related to the notion of the `$\gamma$-exponent' condition that was introduced by \cite{polonik1995measuring}, and widely used in the density level set literature (\cite{tsybakov1997nonparametric,rigollet2009optimal}). It ensures that the conditional density $p(\cdot|x)$ is neither too flat nor too steep near the contour at level $t_x^{\alpha}$, so the cut-off value $t_x^{(\alpha)}$ and the conditional density level set $\mathbf{C}_P(x)=L_x(t_x^{\alpha})$ can be approximated from a finite sample [...]. Assumption H4 also requires that the optimal cut-off values $t_x^{(\alpha)}$ be bounded away from zero.". Assumption H5 is required to get the uniform convergence of the Kernel based estimator of $p(y|A_k)$ on manifolds, see \cite{berry2017density}.
These assumptions play a crucial role in the development and analysis of conformal prediction methods.

Let us introduce some important sequences of positive real numbers that will play a key rol all along the manuscript. 
\begin{enumerate}
	\item $w_n=(\log(n)/n)^{1/(d+2)}$,
	\item $\gamma_n=\lfloor b_1nw_n^d/2\rfloor$, being $b_1$ the positive constant introduced in H2. Observe that $\gamma_n\to\infty$.
	\item  Given $h_n,c_n\to 0$ as $n\to\infty$, we will consider the subsquences $c_{\gamma_n}$, and $h_{\gamma_n}$ being $\gamma_n=\lfloor b_1nw_n^d/2\rfloor$ as before.
\end{enumerate}

\section{Locally valid sets from a kernel density estimator}

In this section  we introduce a slightly modified version of the estimated local marginal density $\hat{p}^{(x,y)}(v|A_k)$ and of the local conformity rank $\pi_{n,k}(x,y)$ originally introduced in \cite{waser14}. We make the assumption that $  \textrm{supp}(P_X)=[0,1]^d$, and that $\mathcal{A}=\{A_k,k=1,\dots,T\}$ is a finite partition of $[0,1]^d$ consisting of equilateral cubes with sides of length $w_n$. This is a quite common and technical assumption in conformal inference, but we can assume that,  for instance, $X$ is such that $\textrm{supp}(P_X)\subset [-R,R]^d$, for  some $R>0$. It only changes the constants appearing in Theorem  \ref{thmain}, which remains true.

Let $n_k=\sum_{i=1}^n \I _{\{X_i\in A_k\}}$. Given a sequence $h_n\to 0$, a kernel function $K(\cdot):\mathbb{R}\to \mathbb{R}$, and $h_{n_k}$, we define
\begin{equation}\label{kernel}
	\hat{p}(v|A_k)=\frac{1}{n_kh_{n_k}^{\ell}}\sum_{i=1}^n \I_{\{X_i\in A_k\}} K\Bigg(\frac{||Y_i-v||}{h_{n_k}}\Bigg).
\end{equation}

We aim to prove that $\hat{p}(v|A_k)$ provides a uniform estimate of $p(y|A_k)$ across $v$ and $k$. To achieve this, we need to ensure that there are sufficiently many sample points in each $A_k$. This is guaranteed by lemma 9 of \cite{waser14}, which states that if we  choose $w_n=(\log(n)/n)^{1/(d+2)}$, then, with probability one, for all $n$ large enough, 

\begin{equation}\label{nk}
	\forall k:\quad b_1nw_n^d/2\leq n_k\leq 3b_2nw_n^2/2.
\end{equation}
Recall that $b_1$ and $b_2$ were defined in Assumption H2. In what follows, we assume that $n$ is sufficiently large to ensure  \eqref{nk}.

The corresponding augmented estimate, based on $\aleph_n\cup\{ (x,y)\}$ is, for any $(x,y)\in A_k\times \mathcal{M}$, 
$$\hat{p}^{(x,y)}(v|A_k)=\frac{n_k}{n_k+1}\hat{p}(v|A_k)+\frac{1}{(n_k+1)h_{n_k}} K\Bigg(\frac{||y-v||}{h_{n_k}}\Bigg).$$
For any $(X_{n+1},Y_{n+1})=(x,y)\in A_k\times \mathcal{M}$, consider the following \textbf{local conformity rank}
\begin{equation}\label{locconfrank}
	\pi_{n,k}(x,y)=\frac{1}{n_k+1}\sum_{i=1}^{n+1} \I_{\{X_i\in A_k\}}\I_{\{\hat{p}^{(x,y)}(Y_i|A_k)\leq \hat{p}^{(x,y)}(Y_{n+1}|A_k)\}}.
\end{equation}

As  proved in Proposition 2 of  \cite{waser14}, the set
\begin{equation}\label{estimator}
	\hat{\mathbf{C}}_n(x)=\{y:\pi_{n,k}(x,y)\geq \alpha\},
\end{equation}
has finite sample local validity, i.e., it satisfies \eqref{localval}. The finite sample local validity, as established by \eqref{estimator}, does not require any assumptions and it holds under very general conditions.

The following result is the key theorem. Its proof is deferred to the Appendix. The proof  follows the ideas used to prove Theorem 1 of \cite{chola21}. It states that $p(y|A_k)$ can be estimated uniformly by \eqref{kernel} for all $y$ that are far enough from the boundary of $\mathcal{M}$. Additionally, this estimation can be made uniformly across all $k$. We assume, as in \cite{chola21}, that $K$ is a Gaussian kernel. This restriction is, as in \cite{chola21}, purely technical.    More recent  references on kernel-density estimation on manifolds are \cite{Bouzebda:2024,BouzebdaTaachouche2024}, \cite{Cleant,Cleant22,Beren21,WuWu2022}.

This result is of fundamental importance in conformal prediction, as it provides a way to estimate the conditional density of $Y$ given $A_k$ for any $k$ in a non-parametric way. It implies that the estimation error is uniform across all partitions $A_k$, which is a key requirement for conformal prediction methods. In particular, it allows us to construct conformal prediction regions that are valid with a given level of confidence.

\begin{theorem} \label{th1}
	Let $\mathcal{M}\subset \mathbb{R}^D$ be a compact $\ell$-dimensional manifold satisfying H0. Let $\mathcal{M}_n\subset \mathcal{M}$ be a sequence of closed sets, and let $h_n\to 0$ be a sequence of setwidths such that $nh_n^{\ell+3}/\log(n)\to \infty$. We assume that $c_n:=\inf_{y\in \mathcal{M}_n} \rho(y,\partial \mathcal{M})\to 0$ is such that $h_n/c_n\to 0$ monotonically. Additionally, we assume that $p(y|A_k)$ satisfies H5. Then, we have
	
	\begin{equation}\label{convunif}
		\sup_k \sup_{y\in \mathcal{M}_{n_k}}|\hat{p}(y|A_k)- p(y|A_k)|=o(h_{\gamma_n}/c_{\gamma_n})\quad \text{a.s.},
	\end{equation}
	
	\noindent where $\gamma_n=\lfloor b_1nw_n^d/2\rfloor$.
\end{theorem}

The following theorem is the main result of our paper. It  states that $\hat{\mathbf{C}}_n(X_i)$ consistently estimates $\mathbf{C}_P(X_i)$ uniformly for all $X_i \in \aleph_n$. The proof, which is deferred to the Appendix, see Subsection \ref{proofthmain}, is based on some ideas from the proof of Theorem 1 of \cite{waser14}. Specifically, Lemma \ref{lemaux0} and Lemma \ref{lemaux2} are adapted from \cite{waser14}, while Lemma \ref{cotahatp} is new and is required to adapt the proof of Lemma \ref{lemaux2}. Lemma \ref{lemaux3} is the same as Lemma 8 of \cite{waser14}. Finally, the proof of Theorem \ref{thmain} uses the adapted lemmas and considers separately the sample points that are close to $\mathcal{M}$ and those that are far away from this boundary. The first set of points is shown to be negligible with respect to $\nu$ using techniques from geometric measure theory.

\begin{theorem} \label{thmain} Assume H0 to H5. Let $\hat{\mathbf{C}}_n(x)$ be given by \eqref{estimator} and $\mathbf{C}_P(x)$ be given by \eqref{oracle}. Let $w_n=(\log(n)/n)^{1/(d+2)}$ and $\gamma_n= \lfloor b_1nw_n^d/2\rfloor$. Then, for any $\lambda>0$, there exists an $A_\lambda>0$ such that, for $n$ large enough,
	\begin{equation} \label{thmaineq} \mathbb{P}\Big(\sup_{X:(X,Y)\in \aleph_n}\nu\big(\hat{\mathbf{C}}_n(X)\triangle \mathbf{C}_P(X)\big)>A_\lambda c_{\gamma_n}\Big)=\mathcal{O}(n^{-\lambda}),
	\end{equation}
	where $c_{\gamma_n}$ is such that $h_{\gamma_n}/c_{\gamma_n}^2\to 0$ and $h_{\gamma_n}\to 0$ such that $\gamma_nh_{\gamma_n}^{\ell+3}/\log(\gamma_n)\to \infty$. 
\end{theorem}

\begin{remark} \label{rateth2} It can be easily seen that, up to logarithmic factors, the rate of $c_{\gamma_n}$ is $n^{-1/((\ell+3)(d+2))}$. \end{remark}

\section{Non-compact manifolds}\label{noncomp}

In this section, we discuss how to remove the compactness assumption imposed in Theorem~\ref{thmain}, while retaining all the other hypotheses. We will prove that the set \(\hat{\mathbf{C}}_n(X_{n+1})\) satisfies \emph{asymptotic conditional validity} (see \cite{izbicki2022cd}). In other words, there exist random sets \(\Lambda_n\) such that
\[
\mathbb{P}\Bigl(X_{n+1}\in \Lambda_n | \Lambda_n\Bigr)= 1 - o(1),
\]
and
\[
\sup_{x_{n+1}\in\Lambda_n}\Bigl|\mathbb{P}\Bigl(Y_{n+1}\in \hat{\mathbf{C}}_n(X_{n+1}) \,\Big\vert\, X_{n+1}=x_{n+1}\Bigr) - \bigl(1-\alpha\bigr)\Bigr| = o(1).
\]

According to Theorem~6 in \cite{izbicki2022cd}, to obtain asymptotic conditional validity it suffices to prove that 
\begin{equation}\label{noncompact}
	\mathbb{P}\Bigl(Y_{n+1}\in \hat{\mathbf{C}}_n(X_{n+1})\triangle \mathbf{C}_P(X_{n+1})\Bigr) = o(1).
\end{equation}

\begin{theorem}\label{noncompact2}
	Assume  H0 to H5. Let \(\hat{\mathbf{C}}_n(x)\) be given by \eqref{estimator} and \(\mathbf{C}_P(x)\) by \eqref{oracle}. Then, \(\hat{\mathbf{C}}_n(X_{n+1})\) is asymptotically conditionally valid.
\end{theorem}

It worth to be mentioned that asymptotic conditional validity is considerably weaker than \eqref{thmaineq} when \(p(y| x)\) is bounded from above. 

\section{Computational aspects}\label{highdim}

A different method for estimating the conformal region is introduced in \cite{izbicki2022cd}.
Instead of using a partition of $\mathbb{R}^d$ on cubes to get a partition  of \(X_1,\dots,X_n\), to build kernel-based estimators, this approach
employs a new partition constructed as follows.
Consider, as in \cite{izbicki2022cd}, the functions
\[
H(z| x) = \int_{\{y : p(y| x) \le z\}} p(y| x)\,d\nu(y),
\qquad
\hat{H}(z | x) = \int_{\{y : \hat{p}(y| x) \le z\}} \hat{p}(y | x)\,d\nu(y).
\]
where $\hat{p}(y| x)$ is any estimator of $p(y| x)$, not necessarily kernel-based. Define the conditional \(\alpha\)-quantile of \(Y| X\) by $q_\alpha(x) = H^{-1}(\alpha| x).$ Similarly, let $\hat{q}_\alpha(x) = \hat{H}^{-1}(\alpha | x)$,
be the estimate of the conditional \(\alpha\)-quantile of \(Y| X\).

Let \(\mathcal{F}\) be a partition of \(\mathbb{R}^+\). Then we create a partition
\(\mathcal{A} = \{A_1,\dots,A_T\}\) of \(X_1,\dots,X_n\) by assigning \(X_i\) and \(X_j\) to the
same set \(A_r\in\mathcal{A}\) if and only if \(\hat{q}_\alpha(X_i)\) and \(\hat{q}_\alpha(X_j)\)
lie in the same element of \(\mathcal{F}\). We then build the kernel-based estimator \eqref{kernel}
according to this new partition. This approach is particularly suitable for  high-dimensional feature spaces.

We propose either the set \(\hat{\mathbf{C}}_n(x)\) (as defined in \eqref{estimator}) or one of the algorithms presented in \cite{izbicki2022cd}. According to Theorem 25 in \cite{izbicki2022cd}, to establish the asymptotic conditional validity of this proposal (in the sense of \eqref{noncompact}), it is necessary to assume some smoothness of \(H(z,x)\) (see Assumption 23 in \cite{izbicki2022cd}), as well as the compactness of \(M\) and the existence of sequences \(\eta_n = o(1)\) and \(\rho_n = o(1)\) satisfying
\[
\mathbb{P}\!\Biggl(
\mathbb{E}\Bigg[
\sup_{y \in \mathcal{M}} \bigl(\hat{p}(y|X) - p(y|X)\bigr)^2
\;\Big|\; \hat{p}
\Bigg] 
\;\ge\; \eta_n
\Biggr) 
\;\le\; \rho_n.
\]
This condition indeed holds in our case; however, proving it in detail would require rewriting the proof of Theorem 1 in \cite{chola21}. We leave this verification to the reader.
As in the non-compact case, the consistency result obtained is considerably weaker than \eqref{thmaineq}.

Since the computational burden  of approximating $\hat{\mathbf{C}}_n(x)$ is high, \cite{waser14} proposes replacing that set with $\mathbf{C}_n^+(x)$, see \eqref{cmas}, which, by including $\hat{\mathbf{C}}_n(x)$, has coverage of at least $1 - \alpha$.  Let $A_x\in \mathcal{A}$ the element that contains $x$.

\begin{enumerate}
	\item Let $\hat{p}(x,y)$ be the joint density estimator based on the subsample (of cardinality  $n_x$) of points for which $X_i\in A_x$ 
	\item Let $Z_i = (X_i, Y_i)$ for $i=1,\dots,n_x$, and let 
	$Z_{(1)}, Z_{(2)}, \dots, Z_{(n_x)}$ denote the sample ordered 
	increasingly by $\hat{p}(X_i, Y_i)$.
	\item Let $j = \lfloor n_x \alpha \rfloor$ and define
	\begin{equation}\label{cmas}
		\mathbf{C}_n^+(x) \;=\; \Bigl\{\, y :\; 
		\hat{p}(x,y) \;\ge\; \hat{p}\bigl(X_{(j)}, Y_{(j)}\bigr) 
		\;-\; \frac{K(0)^2}{n_x\,h^{d+1}} 
		\Bigr\}.
	\end{equation}
\end{enumerate}

\section{Simulation examples}

\subsection{Toy model: Output on the sphere}

We consider a regression model with output variable denoted by $Y_i$ defined on the unit sphere $S^2$. The input variable $X_i$ takes values in the interval $[-1, 1]$. The model is given by the following probability distribution:

$$Y_i \sim \mathbb{F} \left( \frac{\eta+ \beta X_i}{ \Vert \eta + \beta X_i \Vert }, \kappa  \right), \, i=1,2 \ldots, 400.$$ 

Here, $\mathbb{F}$ denotes the Von Mises--Fisher distribution (see \cite{mardia2000}), and $X_i \sim \textrm{U}(-1, 1)$ is an i.i.d. sample. For the simulation study, we set $\eta = (1, 0, 0)$, $\beta = (0, 0, 1)$, and $\kappa = 200$.

To estimate the kernel density, we use the estimator given by Equation \eqref{kernel} with a setwidth parameter $h = 0.5$. Additionally, we partition the data into intervals $A_k = \left(-1 + (k-1)/2, -1 + k/2\right)$, where $k = 1, \ldots, 4$.

Using the proposed method, we obtain a $90\%$ confidence set $\hat{\mathbf{C}}_{400}(0)$, shown in green in Figure \ref{bola}. In this particular case, the unobserved output was $Y = (1, 0, 0)$, which is illustrated in purple in Figure \ref{bola}. Points belonging to the boundary of the theoretical confidence set $\ref{oracle}$ for $\alpha = 0.1$ are displayed in blue.

\begin{figure}
	\centering
	\includegraphics[width=65mm]{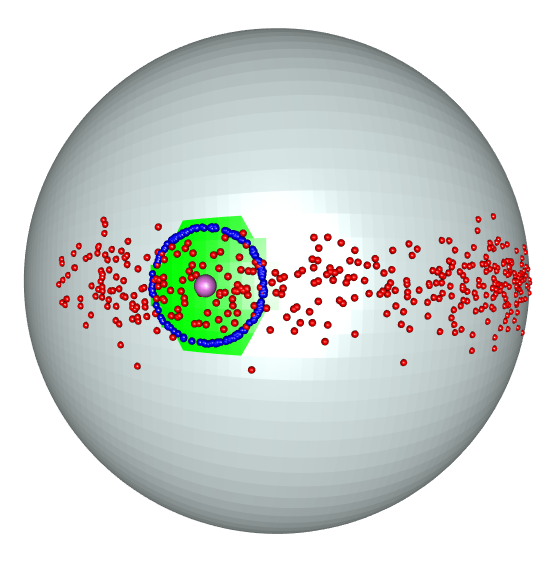 }
	\caption{Red: a sample $Y_1, \ldots, Y_{400}$ on $S^2$. Green: the estimated $90\%$-confidence set for $Y$  if $x=0$. Purple: the point prediction for $x=0$, which is $y=(1,0,0)$.  Blue dots belong to the boundary of the theoretical confidence set \eqref{oracle} for $\alpha=0.1$. }
	\label{bola}
\end{figure}

\subsection{Output in a Stiefel manifold.}

It is common to summarize information from a set of variables using Principal Component Analysis (PCA) or Factor Analysis. This approach is frequently employed in constructing indices within the social sciences, as demonstrated by Vyas (2006). Subsequently, researchers often attempt to explain these indices using other covariates. To illustrate this scenario, we provide a `toy example' using simulations.

We will define a regression model whose output lies on the Stiefel Riemannian manifold $SO(3,2)$, represented by two-dimensional orthonormal vectors in $\mathbb{R}^3$. The construction of the regression is as follows.

The input variable, denoted by $X$, is uniformly distributed on the interval $[0,1]$. For each value of $X$, we sample $400$ points from a 3-dimensional Gaussian distribution centered at the origin $(0,0,0)$. The covariance matrix of this Gaussian distribution has eigenvalues $2$, $1$, and $1/2$, along with the corresponding eigenvectors $(2+X,2+2X,2-X)$, $(-2+X,2-3X,-2+X)$, and $(1,0,0)$, respectively.

After generating these $400$ points for each value of $X$, the output $Y\in SO(3,2)$ is obtained by applying a PCA to these points and retaining only the first two principal directions. To conduct the simulations analysis, we sample $500$ points, denoted by $X_1,\dots, X_{500}$.

In Figure \ref{sti}, we display a sample of $500$ output values, represented as gray arcs. The specific output value $y_{0.1}\in SO(3,2)$, corresponding to the input $X=0.1$, is depicted as a purple dotted arc.

To estimate the  density, we use the estimator given by Equation \eqref{kernel}, with a setwidth parameter $h = 1$. Additionally, we divide the data into intervals $A_k = \left((k-1)/5, k/5\right)$, where $k = 1, \ldots, 5$. It is important to note that the $SO(3,2)$ data are embedded in $\mathbb{R}^6$ ($D=6$), and we consider the Euclidean distance in this space. The dimension of the submanifold in this case is $\ell=3$.

To visualize the confidence set, we draw $10000$ points within $SO(3,2)$. The points falling within the confidence set are highlighted in orange in Figure \ref{sti}.

\begin{figure}
	\centering
	\includegraphics[width=65mm]{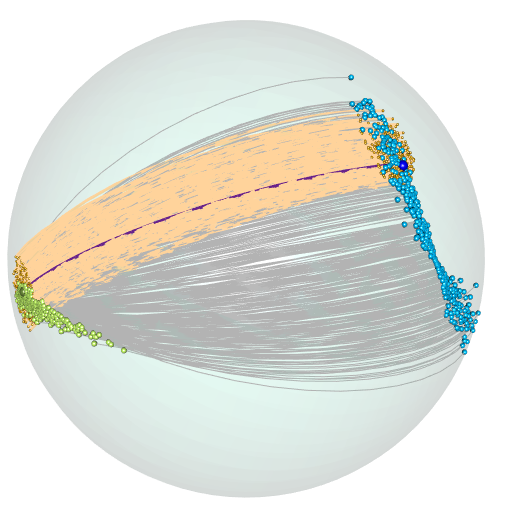 }
	\caption{A uniform sample of 10000 points in the $95\%$ confidence set (orange arcs) for $Y \in SO(3,2)$  if $x=0.1$.  The sample  value for $x=0.1$ is the $3\times 2$ matrix, $y_x$, whose first column is $(0.62,0.5,0.6)$ and second column is $(-0.3,0.86,-0.44)$, which is  depicted in purple. The observations are represented by gray arcs in the sample in the output.}
	\label{sti}
\end{figure}

\section{Real-data examples}

\subsection{Example 1: Regression between cylindrical manifolds with application to wind modeling}
 
In this study  we address a problem of substantial practical relevance for the development of wind energy infrastructure in Uruguay. Specifically, the goal is construct prediction sets for wind intensity \( R \) and direction \( \theta \), at a meteorological station located in Montevideo, based on simultaneous observations from a nearby station in Maldonado. These variables together form a point on a cylinder \( \mathcal{D} = S^1 \times \mathbb{R}^+ \subset \mathbb{R}^3 \), where \( S^1 \) represents the directional component (wind angle) and \( \mathbb{R}^+ \) the non-negative intensity. Thus, the regression problem consists of learning a mapping from one cylindrical manifold to another.

This task is not only of theoretical interest due to the non-Euclidean structure of the input and output spaces, but also of  \textit{strategic importance for national energy planning}. Accurate prediction of wind behavior at candidate sites is a critical step in the optimal placement of wind turbines, which are a major source of renewable energy in Uruguay’s energy matrix.

We focus on moderate and extreme wind regimes (within the range 0 to 20 m/s), using hourly meteorological data recorded during the month of July between 2008 and 2016. The data were collected at two locations: the Laguna del Sauce Weather Station in Maldonado (Station 1) and the Carrasco International Airport Weather Station in Montevideo (Station 2), which are approximately 84 kilometers apart. To ensure reliability, only time points for which both stations reported valid measurements were retained, resulting in a total of 5,362 matched observations.

Our objective is to construct  conformal prediction sets, for wind conditions at Station 2 given observations from Station 1. These sets aim to quantify predictive uncertainty while respecting the cylindrical geometry of the variables involved, offering robust tools for risk assessment and decision-making in wind farm siting.

Figures \ref{corD} and \ref{corI} show that there is a correlation in both wind direction and intensity between the two stations.  The $\xi$-correlation, see \cite{chatterjee2021},  between the intensities is $0.45$. In the case of the directions, the angular correlation is $0.63$.  The angular correlation is the linear correlation between the variables $\sin \left( \theta_1-  \bar{\theta}_1\right)$  and  $ \sin \left( \theta_2-  \bar{\theta}_2 \right)$,  see \cite{jammalamadaka1988}. 

\begin{figure}
	\centering
	\includegraphics[width=65mm]{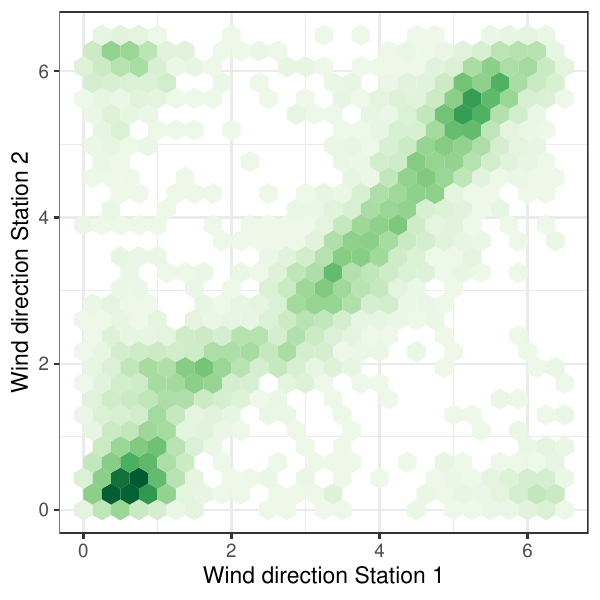 }
	\caption{Scatter plot between the angular directions (in radians) of the winds recorded at Station 1 and Station 2 in the months of July between 2008 and 2016.}
	\label{corD}
\end{figure}

\begin{figure}
	\centering
	\includegraphics[width=65mm]{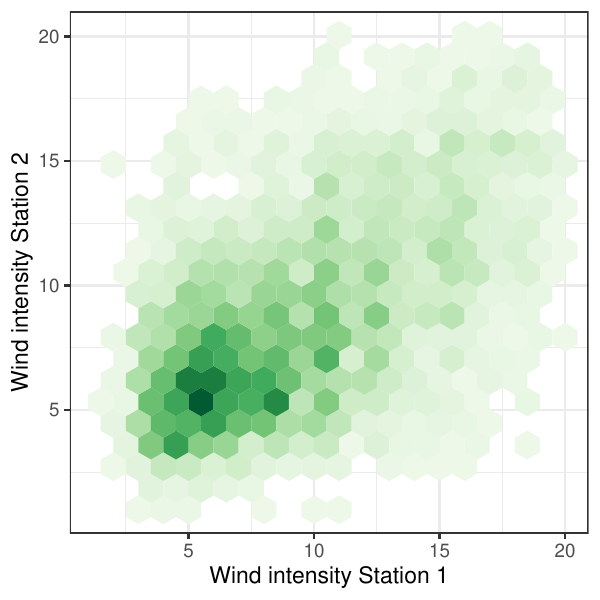 }
	\caption{Scatter plot between the intensities (in m/s) of the winds recorded at Station 1 and Station 2 in the months of July between 2008 and 2016.}
	\label{corI}
\end{figure}

In this example, the partition is $$A_{i,j}= \Big\{ (\theta,R) \in \mathcal{D} : \theta \in [ i \pi/4 , (i+1)\pi/4] \textrm{ and }  R \in [4j , 4(j+1)]  \Big \}$$ with $i=0,1,\ldots,7$ and $j=0,1,\ldots,4$.  For the kernel density estimator we choose $h_k =0.4$.

Figure \ref{cil} shows the confidence set (at 80\%) obtained by our method for $(\theta_1, R_1)= (2.3 \textrm{ radians},  5.1 \textrm{ m/s})$. This was recorded at Station 1 on 2008-07-01 at 7 pm.  On the same date and time the data recorded at Station 2 were $(\theta_2, R_2)= (2.4 \textrm{ radians},  6.6 \textrm{ m/s})$ (point depicted in purple in Figure \ref{cil}). 

\begin{figure}
	\centering
	\includegraphics[width=105mm]{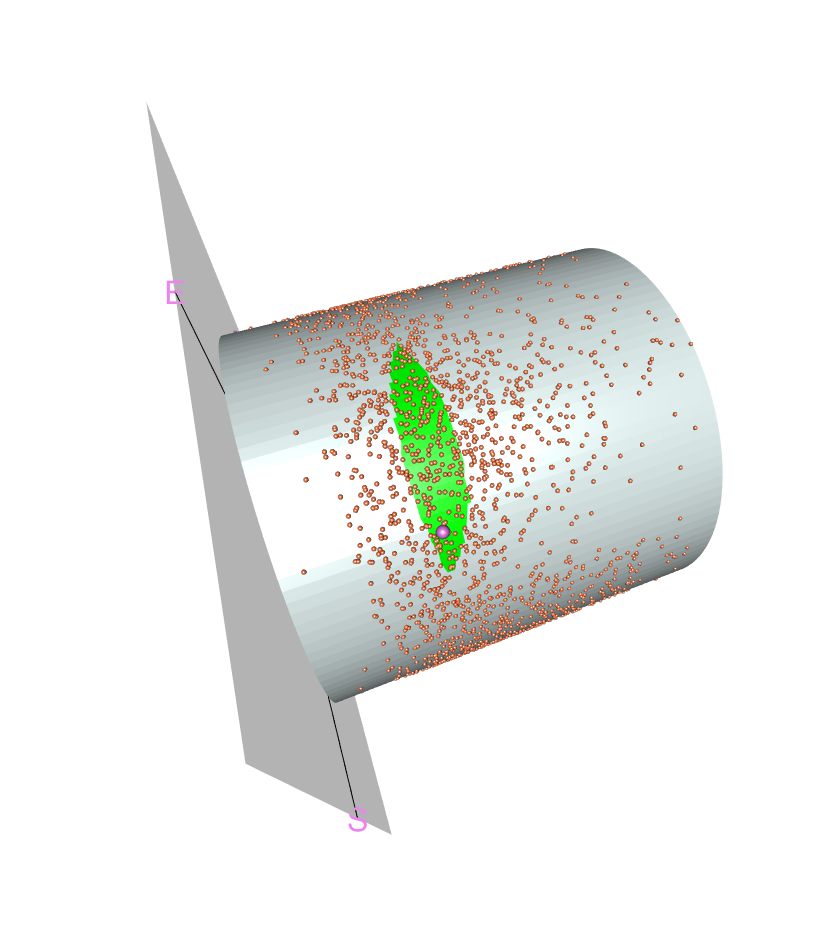 }
	\caption{Winds (not exceeding 20 m/s) at Station 2 in the months of July between 2008 and 2016 (red). The green area is the 80\% confidence set for the data recorded at Station 1 of $(2.3 \textrm{ radians}, 5.1 \textrm{ m/s})$  on 2008-07-01 at 7 pm. The purple dot represents the wind recorded at the same time at Station 2 ($(\theta_2, R_2)= (2.4 \textrm{ radians},  6.6 \textrm{ m/s})$).}
	\label{cil}
\end{figure}

\subsection{Example 2: Regression on the simplex with application to multiclass classification}
 
Regression models whose outputs lie in the $(d-1)$-dimensional unit simplex $\Delta^{d-1}$ play an essential role in many applied contexts (see, for example, \cite{aitchison1986statistical,pawlowsky2015modeling}). In supervised classification, a variety of machine-learning algorithms produce vectors of class-membership probabilities—i.e., elements of $\Delta^{d-1}$—which we denote by $Y_i$. Here, each $Y_i$ is obtained by using Random Forest (RF). Because these vectors reside in $\Delta^{d-1}$, their components naturally estimate the probability of membership in each class. The final class prediction is then chosen as the one with the highest estimated probability, equivalent to identifying the Voronoi cell of the simplex in which $Y_i$ falls (see Figure~\ref{fig:simplex}).

%

Our goal is to construct a \(100(1-\alpha)\%\) confidence region for observations on the simplex \(\Delta^{d-1}\) by adapting Algorithm 1 from \cite{izbicki2022cd} (CD-split) to this setting. In particular, we follow \cite{izbicki2017converting} to estimate the conditional density, employing an appropriate Fourier‐type basis on the simplex—such as the Bernstein polynomial basis \cite{farouki2003construction,ghosal2001convergence}.

The intersection of this confidence region with the Voronoi tessellation of $\Delta^{d-1}$ yields the \emph{conformal class set}, i.e., the subset of classes that are statistically consistent with the observation at level \( \alpha \). Moreover, this approach provides insight into how the predictive uncertainty is distributed among the classes, by analyzing the proportion of the confidence region that falls within each Voronoi cell.

\begin{figure}
	\centering
	\includegraphics[width=80mm]{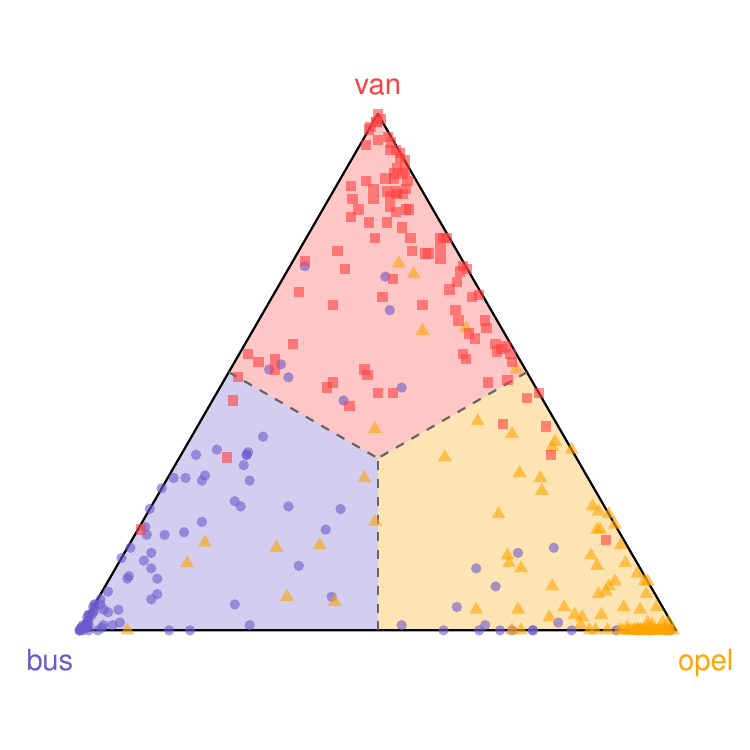 }
	\caption{Test set predictions for vehicle classification. Colors denote the true class labels, and Voronoi regions indicate the predicted class according to Random Forest.
	}
	\label{fig:simplex}
\end{figure}

An application of the proposed methodology is carried out on a real-world dataset, where the covariate space is 5-dimensional and the response variable lies on $\Delta^2$, representing class probabilities, obtained, we it was said, by means of RF. To this end, we consider the publicly available \texttt{Vehicle} dataset from the \texttt{mlbench} package in \texttt{R} \cite{leisch2007mlbench}.

The \texttt{Vehicle} dataset originates from a benchmark study conducted at the Turing Institute in the early 1990s. It was designed to evaluate the performance of classification algorithms in distinguishing between different types of road vehicles based on geometric features extracted from digitized images of their silhouettes. Each observation in the dataset corresponds to a vehicle image and is characterized by $5$ numerical covariates, such as aspect ratio, edge count, and moment-based shape descriptors.

The categorical response variable indicates the type of vehicle. For the purposes of this study, we restrict attention to a subset of the data containing three of the four classes (\texttt{bus}, \texttt{opel}, and \texttt{van}), so that the associated output probabilities lie on $\Delta^2$. This structure is ideal for illustrating our simplex-based predictive modelling approach. The dataset contains a total of $628$ observations, and its moderate size, combined with its multi-class nature, makes it well-suited for benchmarking methods involving compositional or probabilistic responses. 

The dataset is randomly split into two equal parts: 50\% of the observations are used for training the model, and the remaining 50\% are reserved as a test set for evaluating predictive performance.

Figure~\ref{fig:simplex2} displays the test data projected onto $\Delta^2$. Each point corresponds to a predicted probability vector, and its color indicates the true class label. The predicted class for a point is the Voronoi cell that contains the point.

We assess confidence levels of \(90\%\) and \(95\%\) for the covariate vector \(\mathbf{x}_{n+1} = (84, 45, 66, 150, 65)\). The corresponding class probabilities predicted by the Random Forest are \(\mathbf{y}_{n+1} = (0.845, 0.01, 0.145)\), highlighted in purple in Figure~\ref{fig:simplex2}. The resulting class band, restricted to the ``bus’’ and ``van’’ categories, offers a detailed decomposition of predictive uncertainty by indicating the fraction of the confidence region attributable to each class. At the \(95\%\) confidence level, approximately \(93\%\) of the band resides within the ``bus’’ region, while the remaining \(7\%\) falls within the ``van’’ region.

\begin{figure}
	\centering
	\includegraphics[width=65mm]{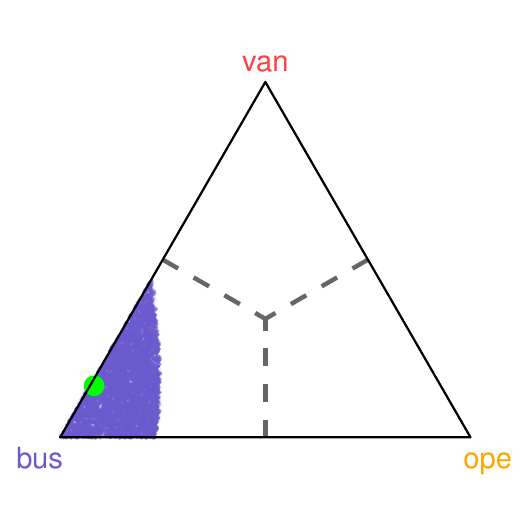 }
	\includegraphics[width=65mm]{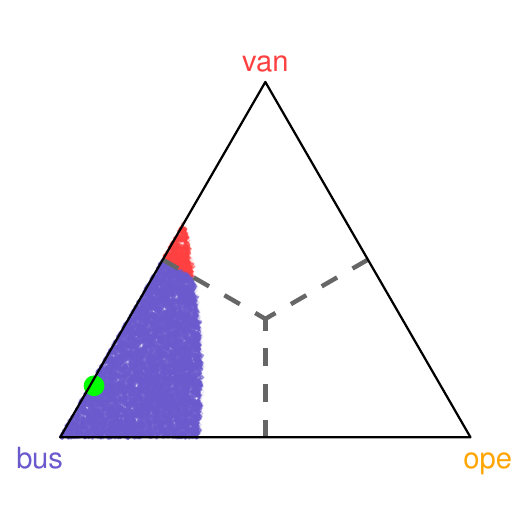 }
	
	\caption{
		Left panel: $90\%$ conformal prediction region; right panel: $95\%$ conformal prediction region for $y_{n+1} = (0.845,\ 0.01,\ 0.145)$ (shown in green). 
		The background colors represent the Voronoi regions of the predicted classes over the 2-dimensional simplex $\Delta^2$. 
		Classification was performed using a Random Forest model.
	}
	\label{fig:simplex2}
\end{figure}

\color{black}
\section{Conclusions and future work}

We work within the conformal‐inference framework, where the response variable takes values on a Riemannian manifold while the covariates lie in \(\mathbb{R}^d\). Under assumptions analogous to those in \cite{waser14}, we extend their results to show that the conditional oracle set \(\mathbf{C}_P(x)\) can be estimated consistently. A central tool in our analysis is the almost‐sure, uniform consistency of the classical kernel density estimator on manifolds, as proved in \cite{chola21}. We also cover the cases of manifolds with boundary and non‐compact manifolds, and we discuss computational strategies for handling high‐dimensional input data.

Our theoretical contributions are demonstrated through two simulation studies and two real-world data examples.

Finally, one could explore alternative nonconformity scores beyond the KDE-based criterion we employ—for example, quantile-based or regression-based scores—each of which demands a nontrivial adaptation of the Euclidean theory to the manifold setting.

\section*{Appendix}

\begin{proof}[Proof of Theorem \ref{th1}]  Since $n_kh_{n_k}^{\ell+1}/\log(n_k)\to \infty$, then, from theorem 1 of \cite{chola21},  for all $\epsilon>0$, with probability one, for all $n$ large enough,
	\begin{equation}\label{estdensi}
		\frac{c_{n_k}}{h_{n_k}}\sup_{y\in \mathcal{M}_{n_k}}|\hat{p}(y|A_k)- p(y|A_k)|\leq \epsilon.
	\end{equation}
	Note that the index $n_0$ for which \eqref{estdensi} holds depends on the choice of $k$. However, from \eqref{nk}, we know that with probability one, $n_k\geq \gamma_n$ for all $k$ when $n$ is sufficiently large. Since $h_n/c_n\to 0$ monotonically, it follows that for all $k$, $c_{n_k}/h_{n_k}\leq c_{\gamma_n}/h_{\gamma_n}$. Hence, we can conclude that \eqref{convunif} holds for all $k$. 
\end{proof}

\subsection{Proof of Theorem \ref{thmain}}\label{proofthmain}

The proof of Theorem \ref{thmain} is based on some technical lemmas. To state them, let us define $L_x(t)$ as the set of $y$ such that $p(y|x) \geq t$, and $L_x^l(t)$ as the set of $y$ such that $p(y|x) \leq t$. We also define $\hat{L}_x(t)$ and $\hat{L}_k^l(t)$ as the corresponding sets for $\hat{p}(\cdot|A_k)$, given by \eqref{kernel}.

\begin{lemma} \label{lemaux0}
	Under the hypotheses of Theorem \ref{th1}, assume also H0 to H3. Let $R_n(x)=\sup_{y\in \mathcal{M}_{\gamma_n}}|\hat{p}(y|A_k)-p(y|x)|$. Then, with probability one, for $n$ sufficiently large, we have
	\begin{equation}\label{lemauxeq1}
		\sup_{x\in \textrm{supp}(P_X)} R_n(x)\leq h_{\gamma_n}/c_{\gamma_n}+Lw_n\sqrt{d},
	\end{equation}
	where $\gamma_n=\lfloor b_1nw_n^d/2\rfloor$ and $\inf_{x\in M_{\gamma_n}} \rho(x,\partial \mathcal{M})=c_{\gamma_n}$.
\end{lemma}

\begin{proof}
	
	From \eqref{nk}, we can assume that $n$ is large enough and such that for all $k$, $b_1nw_n^d\leq n_k\leq b_2nw_n^d$. From \eqref{convunif}, we obtain with probability one for $n$ large enough,
	\begin{equation}\label{lemauxeq0}
		\sup_{y\in \mathcal{M}_{\gamma_n}}|\hat{p}(y|A_k)-p(y|A_k)|\leq h_{n_k}/c_{n_k}.
	\end{equation}
	Using assumption H3 and \eqref{nk}, we have
	\begin{align*}
		\sup_{x\in \textrm{supp}(P_X)} R_n(x) &\leq \sup_{x\in \textrm{supp} (P_X)}\sup_{y\in \mathcal{M}_{\gamma_n}}|\hat{p}(y|A_k)-p(y|x)| \\
		\hspace{-1cm}	&\leq \sup_{x\in \textrm{supp}(P_X)}\sup_{y\in \mathcal{M}_{\gamma_n}}|\hat{p}(y|A_k)-p(y|A_k)| + \sup_{y\in \mathcal{M}_{\gamma_n}}|p(y|A_k)-p(y|x)| \\
		&\leq h_{n_k}/c_{n_k} + Lw_n|x-y|\\
		&\leq h_{\gamma_n}/c_{\gamma_n} + Lw_n\sqrt{d},
	\end{align*}
	where the last inequality follows from \eqref{nk} and the fact that $$\inf_{x\in \mathcal{M}_{\gamma_n}, y\in \mathcal{M}} |x-y|\geq \rho(\mathcal{M}_{\gamma_n},\partial \mathcal{M}).$$ Thus, we have shown \eqref{lemauxeq1}.
\end{proof}

\begin{lemma}\label{cotahatp} Assume H0 to H4. Let $w_n=(\log(n)/n)^{1/(d+2)}$, and assume that $\gamma_n$ and $h_n$ are as in Lemma \ref{lemaux0}. Then, there exists $C>0$ such that with probability one, for $n$ large enough
	$$ \sup_k \vert \hat{p}(\cdot|A_k) \vert_\infty \leq C.$$
\end{lemma}

\begin{proof}
	Recall that we are assuming $K$ to be a Gaussian kernel. We assume that $n$ is large enough so that \eqref{nk} holds. Following the proof of theorem 1 of \cite{chola21}, we define 
	$$W_j(y)=(1/(n_kh_{n_k}^{\ell}))\I_{\{X_j\in A_k\}}K(|Y_j-y|/h{n_k}),$$
	$V_j(y)=W_j(y)-\mathbb{E}(W_j(y))$, and $S_n(y)=\sum_{j=1}^n V_j(y)$. Let $\eta>\ell+1$ and consider a covering ${B(p_1,h^\eta),\dots,B(p_l,h^\eta)}$ of $\mathcal{M}$. Then,
	\begin{multline}
		\sup_{y} |S_n(y)|\leq \max_{1\leq j\leq l} \sup_{y\in B(p_j,h^\eta)} \frac{1}{n_kh_{n_k}^{\ell}}|S_n(y)-S_n(p_j)|+\\
		\max_{1\leq j\leq l} \frac{1}{n_kh_{n_k}^{\ell}}|S_n(p_j)|=I_1+I_2.
	\end{multline}
	Since $K$ is Lipschitz, we have $I_1\leq ch^{\eta-\ell-1}$ for some constant $c>0$. Applying Bernstein's inequality to $S_n(p_j)$ for fixed $p_j$ and $A_k$, we get
	$$\mathbb{P}\Big(\frac{1}{n_kh_{n_k}^{\ell}}|S_n(p_j)|>\epsilon\Big)\leq c_1\exp(-c_2n_kh_{n_k}^{\ell}\epsilon^2),$$
	where $c_1$ and $c_2$ are positive constants.
	
	Then, if $\kappa_n=\lfloor 3b_2nw_n^d\rfloor$,
	$$\mathbb{P}\Big(\sup_k \sup_y |S_n(y)|>\epsilon\Big)\leq c_3w_n^{-d}\exp(-c_3nw_n^{d}h_{\kappa_n}^{\ell}).$$
	By the Borel--Cantelli lemma, it follows that $\sup_k \sup_y |S_n(y)|\to 0$ almost surely.
	
	Next, we will bound $\sup_y \mathbb{E}(V_j(y))$. First, we bound $1/h_{n_k}^{\ell}\leq 1/h_{\gamma_n}^{\ell}$,
	$$\mathbb{E}(V_j(y))=\mathbb{E}[\mathbb{E}(V_j(y)|X)]\leq \int_{[0,1]^d}\int_\mathcal{M} \frac{1}{h_{\gamma_n}^{\ell}} K(\vert z-y \vert/h_{\kappa_n})p(z|x)\nu(dz)p_X(x)dx.$$
	Since $p_X(x)$ and $p(z|x)$ are bounded for all $z$ and $x$, it is enough to bound from above
	$$\sup_y\int_\mathcal{M} \frac{1}{h_{\gamma_n}^{\ell}} K(\vert z-y \vert /h_{\gamma_n})\nu(dz),$$
	which is bounded because we assumed that $K$ is a Gaussian kernel.
\end{proof}

Lemma 6 of \cite{waser14} proves a slightly modified version of the following lemma.

\begin{lemma} \label{lemaux2} Assume H0 to H5. Let $\gamma_n=\lfloor b_1nw_n^d/2\rfloor$, $h_{\gamma_n}\to 0$, and $\mathcal{M}_{\gamma_n}\subset \mathcal{M}$ be a sequence of closed sets such that $\inf_{x\in \mathcal{M}_{\gamma_n}} \rho(x,\partial \mathcal{M})/h_{\gamma_n}\to\infty$. Then, for any $\lambda>0$, there exists $\xi_{2,\lambda}$ such that, for $n$ large enough,
	\begin{equation}\label{lemauxeq2}
		\mathbb{P}\Bigg(\sup_{x\in \aleph_n}V_n(x)>\xi_{2,\lambda}w_n\Bigg)= \mathcal{O}(n^{-\lambda}),
	\end{equation}
	where $V_n(x)=\sup_{t\geq t_0} |\hat{P}(L_x^l(t)|A_k)-P(L_x^l(t)|x)|$. Here, $\hat{P}(\cdot|A_k)$ is the empirical distribution of $Y|X\in A_k$, $t_0$ is given in H4, and $w_n=(\log(n)/n)^{1/(d+2)}$.
\end{lemma}

\begin{proof}
	We will provide a sketch of the proof of this lemma since it follows essentially the same idea used to prove lemma 6 of \cite{waser14}.
	
	Let $x\in A_k$ be fixed. Note that $\{L_x^l(t):t\geq t_0\}$ is a nested class of sets with Vapnik--Chervonenkis dimension 2. Then, for all $B>0$ and $\lambda>0$,
	$$\mathbb{P}\Bigg(\sup_t \Big|\hat{P}(L_x^l(t)|A_k)-P(L_x^l(t)|A_k)\Big|>B \sqrt{w_n}\Bigg)=\mathcal{O}(n^{-\lambda}),$$
	where we used that $w_n=(\log(n)/n)^{1/(d+2)}$.
	On the other hand,
	\begin{multline*}
		|P(L_x^l(t)|A_k)-P(L_x^l(t)|x)|\leq Lw_n\nu(L_x(t))\sqrt{d}\leq Lw_n\nu(L_x(t_0))\sqrt{d}\\
		\leq L\nu(\mathcal{M})\sqrt{d}w_n.
	\end{multline*}
	Let $x'\in A_k$. Then
	$$|\hat{P}(L^l_{x'}|A_k)-P(L^l_{x'}(t)|x')|\leq \vert\hat{p}(\cdot |A_k)\vert_\infty\nu(L_x^l(t)\triangle L^l_{x'}(t))+V_n(x)+|G_x(t)-G_{x'}(t)|$$
	where $G_x(t)=P(L_x^l(t)|x)$.  From (30) in \cite{waser14}, $$|G_x(t)-G_{x'}(t)|\leq c_3w_n^{1\wedge \gamma}.$$
	Here, $c_3$ is a positive constant and $\gamma$ is as in H4. From (29) in \cite{waser14}, $\nu(L_x^l(t)\triangle L^l_{x'}(t))\leq c_4w_n^\gamma$ for some positive constant $c_4$. Lastly,  $\sup_k\vert\hat{p}(\cdot|A_k)\vert_\infty$ is bounded, for $n$ large enough, using Lemma \ref{cotahatp}.
\end{proof}

We recall lemma 8 of \cite{waser14}.

\begin{lemma} \label{lemaux3} Fix $\alpha>0$, $t_0>0$ and $\epsilon>0$. Suppose that $p$ is a density function that satisfies H4, and $\hat{p}$ an estimator such that $\vert\hat{p}-p\vert_\infty<v_1$. Let $\hat{P}$ be a probability measure satisfying $\sup_{t\geq t_0} |\hat{P}(L^l(t))-P(L^l(t))|<v_2$. Define 
	$$\hat{t}^\alpha=\inf \{t\geq 0:\hat{P}(\hat{L}^l(t))\geq \alpha\}.$$
	Assume that $v_1$ and $v_2$ are sufficiently small so that $v_1+c_1^{1/\gamma}v_2^{1/\gamma}\leq t^\alpha-t_0$ and $c_1^{-1/\gamma}v_2^{1/\gamma}\leq \epsilon_0$ where $c_1$ and $\gamma$ are the constants given in H4. Then
	$$|\hat{t}^\alpha-t^\alpha|\leq v_1+c_1^{-1/\gamma}v_2^{1/\gamma}$$
	Moreover, for any $\tilde{t}^\alpha$ such that $|\tilde{t}^\alpha-\hat{t}^\alpha|\leq v_3$, if $2v_1+c_1^{1/\gamma}v_2^{1/\gamma}+v_3\leq \epsilon_0$, then there are constants $\xi_1,\xi_2$ and $\xi_3$ such that $\nu(\hat{L}(\tilde{t}^\alpha)\triangle L(t^\alpha))\leq \xi_1v_1+\xi_2v_2+\xi_3v_3$.	
\end{lemma}

\textit{Proof of Theorem \ref{thmain}}

In what follows we assume that $n$ is sufficiently large to satisfy \eqref{nk}. We will consider a sequence of compact sets $\mathcal{M}_n\subset \mathcal{M}$ such that $\inf_{x\in \mathcal{M}_n}\rho(x,\partial \mathcal{M})=c_n>0$ and $\sup_{x\in\partial \mathcal{M}}\rho(x,\mathcal{M}_n)\leq 2c_n$, where $c_n$ is chosen such that $h_{\gamma_n}/c_{\gamma_n}^2\to 0$.

Throughout the proof we denote by $a$ a generic positive constant.
Write 
$$\nu(\hat{\mathbf{C}}_n(x)\triangle \mathbf{C}_P(x))=\nu(\hat{\mathbf{C}}_n(x)\triangle \mathbf{C}_P(x)\cap \mathcal{M}_{\gamma_n})+\nu\Big(\hat{\mathbf{C}}_n(x)\triangle \mathbf{C}_P(x)\cap (\mathcal{M}\setminus \mathcal{M}_{\gamma_n})\Big)$$
where $\gamma_n$ is as in Lemma \ref{lemaux2} and $\mathcal{M}_{\gamma_n}\subset \mathcal{M}$ is a sequence of closed sets such that $\inf_{x\in \mathcal{M}_{\gamma_n}} \rho(x,\partial \mathcal{M})=c_{\gamma_n}$ satisfies $h_{\gamma_n}/c_{\gamma_n}\to 0$.

We will first prove that  there exists a $a>0$ such that 
\begin{equation}\label{measurebound}
	\nu(\mathcal{M}\setminus \mathcal{M}_{\gamma_n})\leq ac_{\gamma_n}.
\end{equation}
Since $\mathcal{M}$ is $\mathcal{C}^2$, it has positive reach, denoted by $\tau_\mathcal{M}$, as shown in proposition 14 of \cite{thale08}. Let $A_r=\{x\in \mathcal{M}: \rho(x,\partial \mathcal{M})<r\}$. From 
$$\sup_{x\in \partial \mathcal{M}} \rho(x,\mathcal{M}_n)\leq 2c_n,$$ it follows that $\mathcal{M}\setminus \mathcal{M}_{\gamma_n}\subset A_{2c_{\gamma_n}}$.

We write $m_n=2\tau_M\sin(c_{\gamma_n}/\tau_\mathcal{M})$, and from proposition A.1 in \cite{aamari21}, we know that $\Upsilon_{\gamma_n}:=\{B(x,m_n):x\in \partial \mathcal{M}\}$ covers $A_{2\gamma_n}$. Since $\partial \mathcal{M}$ has finite Minkowski content due to its positive reach (see corollary 3 of \cite{ambrosio08}), and $\sin(x)\approx x$ when $x\to 0$, there exists a $a>0$ such that $|\Upsilon_{\gamma_n}|/c_{\gamma_n}^{D-\ell-1}\leq a$ for all sufficiently large $n$, where $|\Upsilon_{\gamma_n}|$ is the $D$-dimensional Lebesgue measure of $\Upsilon_{\gamma_n}$.

Thus, $A_{2c_{\gamma_n}}$ can be covered by at most $ac_{\gamma_n}^{-\ell+1}$ balls of radius $2m_n$ centered at $\partial \mathcal{M}$, and the $\nu$-measure of each of these balls is bounded from above by $Bc_{\gamma_n}^{\ell}$ by corollary 1 of \cite{aaron20}. From this, \ref{measurebound} follows.

Now, to bound $\nu(\hat{\mathbf{C}}_n(x)\triangle \mathbf{C}_P(x)\cap \mathcal{M}_{\gamma_n})$, we follow the approach used in the proof of theorem 1 in \cite{waser14}. We apply Lemma \ref{lemaux3} to the density function $p(y|x)$ and the empirical measure $\hat{P}(\cdot|A_k)$, as well as the estimated density function $\hat{p}(\cdot|A_k)$. Here, we provide a sketch of the main changes made to the proof. We denote by $\hat{L}$ the upper level set of $\hat{p}(\cdot|A_k)$.

Let 
$\{i_1,\dots,i_{n_k}\}=\{i:1\leq i \leq n, X_i\in A_k\}$. From lemma 3 in \cite{waser14}, conditioning on $(i_1,\dots,i_{n_k})$, 
$$\hat{L}\Big\{\hat{p}(X_{(i_\alpha)},Y_{(i_\alpha)}|A_k)\Big\}\subset \hat{\mathbf{C}}_n(x)\subset \hat{L}\Big\{\hat{p}(X_{(i_\alpha)},Y_{(i_\alpha)}|A_k)-(n_kh_{n_k})^{-1}\psi_K\Big\},$$
where $\psi_K=\sup_{x,x'}|K(x)-K(x')|$ with $(X_{(i_\alpha)},Y_{(i_\alpha)})$ is the element of 
$$\{(X_{1_1},Y_{i_1}),\dots,(X_{i_{n_k}},Y_{i_{n_k}})\}$$
such that $\hat{p}(Y_{i_\alpha}|A_k)$ ranks $\lfloor n_k\alpha\rfloor$. Let $\hat{t}^\alpha=\hat{p}(X_{(i_\alpha)},Y_{(i_\alpha)})$. It is easy to check that
$$\hat{t}^\alpha=\inf\{t\geq 0:\hat{P}(L^l(t)|A_k)\geq \alpha\}.$$
Consider the event 
$$E=\Big\{\sup_x R_n(x)\leq  2\frac{h_{\gamma_n}}{c_{\gamma_n}},\ \sup_x V_n(x)\leq \xi_{2,\lambda}w_n\Big\}.$$
From $nw_n^dh_{\gamma_n}^{\ell+3}/\log(\gamma_n)\to \infty$, it follows that $w_n/h_{\gamma_n}\to 0$. Then from Lemmas \ref{lemaux0} and \ref{lemaux2}, $\mathbb{P}(E^c)=\mathcal{O}(n^{-\lambda})$. Let $r_n:= h_{\gamma_n}/c_{\gamma_n}$. Since $w_n/h_{\gamma_n}\to 0$, then the event
$$E_1=\Big\{\sup_x R_n(x)\leq  2r_n,\sup_x V_n(x)\leq \xi_{2,\lambda}r_n\Big\}$$
is such  that for all $n$ large enough, $\mathbb{P}(E_1^c)=\mathcal{O}(n^{-\lambda})$. From Lemma \eqref{lemaux3} with $v_1=2r_n$ and $v_2=\xi_{2,\lambda}r_n$, we obtain that, for $n$ large enough,

\begin{equation}\label{lasteq1} \mathbb{P}\Big(\sup_{x} \nu\Big(\hat{L}(\hat{t}^\alpha)\triangle L_x(t^\alpha)\cap \mathcal{M}_{\gamma_n}\Big)\geq \xi_\lambda r_n \Big)=\mathcal{O}(n^{-\lambda}).
\end{equation}
Let $\tilde{t}^{\alpha}=\hat{t}^\alpha-(\gamma_n h_{\gamma_n})^{-1}\psi_K$ and $v_3=(\gamma_n h_{\gamma_n})^{-1}\psi_K$. From $nw_n^dh_{\gamma_n}^{\ell+3}/\log(\gamma_n)\to \infty$, it follows that $v_3\to 0$. Applying Lemma \ref{lemaux3}, 
we get that
$$\mathbb{P}\Big(\sup_{x} \nu\Big(\hat{L}(\tilde{t}^\alpha)\triangle L_x(t^\alpha)\cap \mathcal{M}_{\gamma_n}\Big)\geq \xi^1_\lambda r_n+\xi^2_\lambda v_3 \Big)=\mathcal{O}(n^{-\lambda}).$$
Since $v_3\leq r_n$ for $n$ large enough, we get that
\begin{equation}\label{lasteq2}
	\mathbb{P}\Big(\sup_{x} \nu\Big(\hat{L}(\tilde{t}^\alpha)\triangle L_x(t^\alpha)\cap \mathcal{M}_{\gamma_n}\Big)\geq 2\xi^1_\lambda r_n \Big)=\mathcal{O}(n^{-\lambda}).
\end{equation}
Lastly, we write
\begin{multline*}
	\nu(\hat{\mathbf{C}}_n(x)\triangle \mathbf{C}_P(x)\cap \mathcal{M}_{\gamma_n})\leq\nu\Big(\hat{L}(\hat{t}^\alpha)\triangle L_x(t^\alpha)\cap \mathcal{M}_{\gamma_n}\Big)+
	\\\nu\Big(\hat{L}(\tilde{t}^\alpha)\triangle L_x(t^\alpha)\cap \mathcal{M}_{\gamma_n}\Big)
\end{multline*}
and then the theorem follows from \eqref{lasteq1} and \eqref{lasteq2}.	\\

\textit{Proof of Theorem \ref{noncompact2}}

To prove \eqref{noncompact}, fix any $\epsilon>0$ and choose a sufficiently large compact smooth submanifold $\mathcal{K} = \mathcal{K}(\epsilon) \subset \mathcal{M}$ fulfilling H0, such that 
$	\mathbb{P}(Y \in \mathcal{K}^c) < \epsilon.$
The existence of such a submanifold follows from the existence of smooth exhaustion functions (see Proposition 2.28 in \cite{lee2013introduction}). Then, by a simple union bound,
\begin{multline*}
	\mathbb{P}\Bigl(Y_{n+1}\in \hat{\mathbf{C}}_n(X_{n+1})\triangle \mathbf{C}_P(X_{n+1})\Bigr)
	\le \\ \mathbb{P}\Bigl(Y_{n+1}\in \Bigl(\hat{\mathbf{C}}_n(X_{n+1})\triangle \mathbf{C}_P(X_{n+1})\Bigr) \cap \mathcal{K}\Bigr) + \epsilon.
\end{multline*}	
Now, fixing $\mathcal{K}$, we take—as in the proof of Theorem \ref{thmain}—a sequence of closed sets $\mathcal{K}_{\gamma_n}\subset \mathcal{K}$ such that 
\[
\inf_{x\in \mathcal{K}_{\gamma_n}} \rho(x,\partial \mathcal{K}) = c_{\gamma_n}\to 0 \quad \text{and} \quad \frac{h_{\gamma_n}}{c_{\gamma_n}} \to 0 \quad \text{ as } n \to \infty.
\]
Then we decompose $\mathbb{P}\Bigl(Y_{n+1}\in \hat{\mathbf{C}}_n(X_{n+1})\triangle \mathbf{C}_P(X_{n+1})\cap \mathcal{K}\Bigr) $ equals
\begin{multline*}
	\mathbb{P}\Bigl(Y_{n+1}\in \bigl(\hat{\mathbf{C}}_n(X_{n+1})\triangle \mathbf{C}_P(X_{n+1})\bigr)\cap \mathcal{K}_{\gamma_n}\Bigr)
	+\\ \mathbb{P}\Bigl(Y_{n+1}\in \bigl(\hat{\mathbf{C}}_n(X_{n+1})\triangle \mathbf{C}_P(X_{n+1})\bigr)\cap \bigl(\mathcal{K}\setminus \mathcal{K}_{\gamma_n}\bigr)\Bigr).
\end{multline*}

Since $\mathcal{K}$ is compact and, for all $x$, $p(y| x)$ is a continuous function of $y$, for $n$ large enough we have
\[
\mathbb{P}\Bigl(Y_{n+1}\in \bigl(\hat{\mathbf{C}}_n(X_{n+1})\triangle \mathbf{C}_P(X_{n+1})\bigr)\cap \bigl(\mathcal{K}\setminus \mathcal{K}_{\gamma_n}\bigr)\Bigr) < \epsilon.
\]
Finally, the proof that
\[
\mathbb{P}\Bigl(Y_{n+1}\in \bigl(\hat{\mathbf{C}}_n(X_{n+1})\triangle \mathbf{C}_P(X_{n+1})\bigr)\cap \mathcal{K}_{\gamma_n}\Bigr) \to 0
\]
follows by using the continuity of $p(y| x)$, that $\mathcal{K}$ is in the hypotheses of Theorem \ref{thmain}, and the fact established in the proof of Theorem \ref{thmain} that 
\[
\nu\Bigl(\hat{\mathbf{C}}_n(x)\triangle \mathbf{C}_P(x)\triangle \mathcal{K}_{\gamma_n}\Bigr) \to 0
\]
as $n\to \infty$.
\section*{Acknowledgment}

Authors are grateful with Tyrus Berry, for his insightful comments on the results obtained in his work with Timothy Sauer. 
. The research of the first and third authors has been partially supported by grant FCE-3-2022-1-172289 from ANII (Uruguay), 22MATH-07 form MATH – AmSud (France-Uruguay) and 22520220100031UD from CSIC (Uruguay).

\end{document}